\documentclass[11pt]{article}
\usepackage{amsmath,amsfonts,amssymb,amsthm}
\usepackage{algorithm}
\usepackage{algorithmic}
\usepackage{graphicx}
\usepackage{booktabs}
\usepackage{cite}
\usepackage[margin=1in]{geometry}
\usepackage{url}
\usepackage{setspace}
\usepackage{parskip}
\usepackage{fancyhdr}

\newtheorem{theorem}{Theorem}
\newtheorem{definition}{Definition}

\newtheorem{corollary}{Corollary}

\begin{document}

\title{Optimal Information Combining for Multi-Agent Systems Using Adaptive Bias Learning\thanks{Submitted for Publication to IEEE Transactions on Information Theory}}

\author{Siavash M. Alamouti and Fay Arjomandi}

\date{}

\maketitle
\thispagestyle{empty} % No header on title page

\begin{abstract}
Modern multi-agent systems ranging from sensor networks monitoring critical infrastructure to crowdsourcing platforms aggregating human intelligence can suffer significant performance degradation due to systematic biases that vary with environmental conditions. Current approaches either ignore these biases, leading to suboptimal decisions, or require expensive calibration procedures that are often infeasible in practice. This performance gap has real consequences: inaccurate environmental monitoring, unreliable financial predictions, and flawed aggregation of human judgments. This paper addresses the fundamental question: when can we learn and correct for these unknown biases to recover near-optimal performance, and when is such learning futile? 

We develop a theoretical framework that decomposes biases into learnable systematic components and irreducible stochastic components, introducing the concept of learnability ratio as the fraction of bias variance predictable from observable covariates. This ratio determines whether bias learning is worthwhile for a given system. We prove that the achievable performance improvement is fundamentally bounded by this learnability ratio, providing system designers with quantitative guidance on when to invest in bias learning versus simpler approaches. We present the Adaptive Bias Learning and Optimal Combining (ABLOC) algorithm, which iteratively learns bias-correcting transformations while optimizing combination weights through closed-form solutions, guaranteeing convergence to these theoretical bounds. Experimental validation demonstrates that systems with high learnability ratios can recover significant performance (we achieved 40\%-70\% of theoretical maximum improvement in our examples), while those with low learnability show minimal benefit, validating our diagnostic criteria for practical deployment decisions.
\end{abstract}

\section{Introduction}

Multi-agent information systems are increasingly critical to modern society. Sensor networks monitor air quality in cities, affecting public health decisions. Ensemble models drive billions of dollars in financial trades. Crowdsourcing platforms shape the training data for artificial intelligence systems that influence hiring, lending, and criminal justice decisions. The accuracy of these systems directly impacts human welfare, economic efficiency, and social health.

Yet these systems consistently underperform their theoretical potential due to a pervasive but poorly understood problem: systematic biases that vary with environmental and operational conditions. A temperature sensor's readings drift with humidity levels. A financial model's predictions skew during market volatility. A human annotator's judgments shift with fatigue and task complexity. These biases are not random errors that average out. They are systematic distortions that compound and corrupt the combined estimate.

The cost of ignoring these biases is substantial. Environmental monitoring systems produce false alarms or miss critical events. Financial prediction ensembles lose millions through correlated errors during market stress. Crowdsourced datasets embed biases that propagate through machine learning pipelines, affecting millions of downstream decisions. Current solutions require frequent recalibration, redundant sensors, or simply accepting degraded performance and are often expensive, impractical, or inadequate.

This raises a fundamental question: Can we learn these unknown bias patterns from data and correct for them adaptively? More importantly, when is such learning worthwhile, and when are we better off with simpler approaches?

The answer is not obvious. Bias patterns might be too complex, too random, or too data-starved to learn effectively. The computational cost might outweigh the accuracy gains. The system might lack the right observables to predict bias behavior. Without theoretical guidance, we often resort to trial and error, wasting resources on futile bias learning attempts or missing opportunities for significant improvements.

This paper provides the theoretical foundation and practical tools to answer these questions definitively. Our key insight is that not all bias is learnable. Environmental factors create both systematic patterns predictable from observable conditions and random fluctuations that cannot be anticipated. The ratio between these components, which we term the learnability ratio, determines the fundamental limit on achievable performance improvement.

The theoretical framework for optimal combining presented here draws inspiration from the Alamouti code \cite{1}, which provides a simple yet optimal method for combining signals from multiple transmit antennas. Just as the Alamouti code enables maximum-likelihood decoding through orthogonal space-time coding, our ABLOC algorithm seeks to optimally combine information from multiple agents while correcting for systematic biases. The fundamental principle that intelligent combination of diverse sources can dramatically improve system performance extends naturally from the multiple-antenna wireless channels to the multi-agent information systems we consider here.

The practical importance of this work is underscored by the rise of hybrid edge-cloud (HEC) \cite{2} and Device-First Continuum AI (DFC-AI) architectures \cite{3}, where agents reside directly on end devices and their insights can be combined anywhere in the AI continuum from end devices to cloud servers depending on application requirements. In such systems, multiple agents operating on diverse devices process observations locally, with the flexibility to combine results at any point in the continuum based on latency, bandwidth, and accuracy requirements. Our proposed framework provides the theoretical foundation for understanding when and how to optimally combine these distributed agent observations while accounting for the heterogeneous biases that arise from different operating conditions, hardware capabilities, and environmental factors.

Consider a multi-agent system where $K$ agents provide observations of a common parameter $\theta \in \mathbb{R}^d$ over time $t = 1, \ldots, T$. Each agent $i$ observes:
\begin{equation}
Y_{i,t} = \theta_t + b_i(X_{i,t}) + \varepsilon_{i,t}
\end{equation}
where $b_i: \mathbb{R}^{p_i} \rightarrow \mathbb{R}^d$ represents an unknown bias function depending on observable covariates $X_{i,t} \in \mathbb{R}^{p_i}$, and $\varepsilon_{i,t} \sim \mathcal{N}(0, \sigma_i^2 I_d)$ denotes measurement noise.

The covariates $X_{i,t}$ capture environmental or operational conditions that influence bias. In sensor networks, these might include temperature, humidity, or electromagnetic interference. In financial prediction, they could represent market volatility, trading volume, or macroeconomic indicators. In crowdsourcing, they might encode task difficulty, annotator experience, or time of day. The key assumption is that while the bias functions $b_i$ are unknown, the covariates that influence them are observable.

If the bias functions were known perfectly, optimal estimation would be straightforward through bias correction followed by inverse-variance weighting. However, in practice, these bias functions must be learned from data, raising fundamental questions: What fraction of the performance gap between naive estimation and perfect bias correction can be recovered through adaptive learning? What are the fundamental limits on achievable performance? When is bias learning worthwhile given finite data and computational resources?

This paper provides answers to these questions through three main contributions. First, we establish a theoretical framework that decomposes biases into learnable and unlearnable components, proving that achievable performance improvement is fundamentally bounded by the fraction of bias variance that is predictable from covariates. This bound is tight and represents a fundamental limit on what any bias learning algorithm can achieve. Second, we develop the Adaptive Bias Learning and Optimal Combining (ABLOC) algorithm that provably converges toward these theoretical bounds through an iterative procedure combining bias learning with optimal weight selection. Third, we provide experimental validation on synthetic data with carefully controlled learnability ratios and complete implementation details to ensure reproducibility.

The theoretical framework builds on classical results in estimation theory \cite{4,5} and information fusion \cite{6,7}, extending them to handle unknown bias functions. The algorithmic approach combines ideas from kernel methods \cite{8,9} and convex optimization \cite{10,11}. The experimental methodology employs synthetic data with controlled properties to validate theoretical predictions while providing sufficient detail for independent reproduction.

\section{Related Work}

\subsection{Multi-Source Information Fusion}

The problem of combining information from multiple sources has been studied extensively across different communities. The fundamental limits of distributed estimation have been characterized through information-theoretic analysis \cite{12,13,14}. Classical weighted least squares approaches \cite{15,16} provide optimal combining when error characteristics are known, while consensus algorithms enable distributed implementation \cite{17,18}. These methods typically assume either unbiased sources or known bias models.

The digital communications literature has provided important insights into optimal combining strategies. The Alamouti code \cite{1} demonstrates how orthogonal designs can achieve optimal diversity gains in multiple-antenna systems. These principles of diversity combining and maximum-ratio combining \cite{19,20} inspire our approach to multi-agent information fusion, though our setting requires learning unknown bias functions rather than channel estimation.

Recent work addresses robust fusion with bounded uncertainties \cite{21,22} but does not handle systematic covariate-dependent biases. Methods for Byzantine-robust aggregation \cite{23,24} focus on adversarial corruptions rather than systematic bias correction. Our work differs by explicitly modeling and learning bias functions from covariates.

\subsection{Bias Correction and Calibration}

The sensor calibration literature addresses bias estimation through various approaches. Blind calibration methods \cite{25,26} estimate biases without ground truth but assume simple bias models. Online calibration techniques \cite{27} adapt to drift but require periodic access to reference standards. Cross-calibration approaches \cite{28,29} leverage redundancy but assume spatially uniform phenomena.

Transfer learning and domain adaptation methods \cite{30,31} address related problems of distribution shift but focus on single-source scenarios. Covariate shift correction \cite{32,33} handles changes in input distribution rather than systematic biases. Our approach learns nonparametric bias functions from covariates without requiring ground truth or reference standards.

\subsection{Theoretical Foundations}

Information-theoretic bounds for estimation have been established through the Cramér-Rao inequality \cite{34,35} and its extensions \cite{36,37}. The Fisher information matrix provides fundamental limits for unbiased estimation \cite{38,39} but does not directly address biased estimators with learnable correction.

Recent work on biased estimation \cite{40,41} establishes mean squared error bounds but assumes known bias characteristics. Information-theoretic analysis of distributed learning \cite{42,43} provides communication-computation tradeoffs but does not address bias learning. Our theoretical contribution establishes tight bounds for the case where biases must be learned from finite data.

\section{Theoretical Framework}

\subsection{Bias Decomposition and Learnability}

The fundamental insights underlying our framework are threefold. First, substantial performance improvements over simple averaging are theoretically possible when agents have systematic, covariate-dependent biases. Our theoretical bounds show that improvements up to 80\% in mean squared error are achievable under favorable conditions. Second, these improvements are only accessible when biases contain learnable patterns. Random fluctuations cannot be corrected regardless of algorithmic sophistication. Third, we can determine a priori whether bias learning will help by computing the learnability ratio which is the fraction of bias variance that is predictable from observable conditions. This ratio provides quantitative guidance on when to invest in bias learning versus using simpler methods.

Our ABLOC algorithm demonstrates this potential by achieving 40\%-70\% of the theoretical maximum (translating to 30\%-50\% MSE reduction in our experiments), suggesting room for further algorithmic improvements. The gap between achieved and theoretical performance stems from finite-sample effects, regularization necessary for stability, and the simplicity of our linear bias models. More sophisticated algorithms might close this gap further, but our results establish both the existence of substantial gains and a practical path to achieving them.

\begin{definition}[Bias Decomposition]
\label{def:bias_decomp}
For each agent $i \in \{1, \ldots, K\}$, the bias function can be decomposed as:
\begin{equation}
b_i(X) = f_i(X) + \nu_i
\end{equation}
where $f_i: \mathbb{R}^{p_i} \rightarrow \mathbb{R}^d$ is a deterministic function representing the systematic component of bias that can be learned from covariates, and $\nu_i$ is a zero-mean random variable representing the stochastic component with $\mathbb{E}[\nu_i \mid X] = 0$ and $\text{Var}(\nu_i) = \tau_i^2 I_d$.

Here, \emph{covariates} refer to observed input variables or features contained in $X$ that are not of direct interest but may influence the outcome of the bias function. These variables help model and explain the systematic patterns in $b_i(X)$ that are consistent across observations.
\end{definition}

This decomposition is always valid by construction, as we can define $f_i(X) = \mathbb{E}[b_i | X]$ and $\nu_i = b_i - \mathbb{E}[b_i | X]$. The systematic component $f_i(X)$ captures predictable bias patterns, while $\nu_i$ represents irreducible randomness.

\begin{definition}[Learnability Ratio]
\label{def:learnability}
For agent $i$, the learnability ratio is:
\begin{equation}
\lambda_i = \frac{\|f_i\|^2}{\|f_i\|^2 + \tau_i^2}
\end{equation}
where $\|f_i\|^2 = \mathbb{E}_X[f_i(X)^T f_i(X)]$ under the covariate distribution.
\end{definition}

The learnability ratio $\lambda_i \in [0, 1]$ quantifies the fraction of bias variance that is theoretically learnable from covariates. When $\lambda_i$ approaches 1, the bias is almost entirely systematic and predictable. When $\lambda_i$ approaches 0, the bias is dominated by random fluctuations.

\subsection{Weight Formulation}

To ensure convex optimization with guaranteed convergence, we employ combination weights with closed-form solutions.

\begin{definition}[Weighted Combination]
\label{def:weights}
The combined estimate is:
\begin{equation}
\hat{\theta}_t = \sum_{i=1}^K w_i \tilde{Y}_{i,t}
\end{equation}
where $w_i \in \mathbb{R}$ are weights satisfying $\sum_{i=1}^K w_i = 1$ and $w_i \geq 0$, and $\tilde{Y}_{i,t} = Y_{i,t} - \hat{f}_i(X_{i,t})$ are bias-corrected observations.
\end{definition}

This weight formulation leads to a convex optimization problem with closed-form solution.

\subsection{Performance Bounds}

We now establish fundamental limits on achievable performance improvement through bias learning with scalar weights.

\begin{theorem}[Achievable Performance Bound]
\label{thm:main_bound}
Consider $K$ agents with learnability ratios $\{\lambda_i\}$, measurement noise variances $\{\sigma_i^2\}$, and total bias variances $\{\beta_i^2\}$ where $\beta_i^2 = \|f_i\|^2 + \tau_i^2$. For any bias learning algorithm using $N$ samples, the relative improvement in mean squared error is bounded by:
\begin{equation}
\eta = \frac{\text{MSE}_{\text{baseline}} - \text{MSE}_{\text{achieved}}}{\text{MSE}_{\text{baseline}}} \leq \frac{\sum_{i=1}^K w_i^* \lambda_i \beta_i^2}{\sum_{i=1}^K w_i^* (\beta_i^2 + \sigma_i^2)} + O(N^{-1/2})
\end{equation}
where $w_i^*$ are the optimal weights and the $O(N^{-1/2})$ term represents finite-sample effects.
\end{theorem}

\begin{proof}
We establish the best achievable mean squared error after bias learning through a sequence of steps.

\textbf{Step 1: Residual error decomposition.}
For agent $i$, after learning an estimate $\hat{f}_i$ of the bias function, the residual error is:
\begin{equation}
Y_i - \hat{f}_i(X_i) - \theta = (f_i(X_i) - \hat{f}_i(X_i)) + \nu_i + \varepsilon_i
\end{equation}

The three terms represent: (1) bias estimation error, (2) unlearnable stochastic bias, and (3) measurement noise.

\textbf{Step 2: Variance of residual error.}
Under the assumption that the estimation error, stochastic bias, and measurement noise are uncorrelated:
\begin{equation}
\text{Var}(Y_i - \hat{f}_i(X_i) - \theta) = \text{Var}(f_i - \hat{f}_i) + \text{Var}(\nu_i) + \text{Var}(\varepsilon_i)
\end{equation}

\textbf{Step 3: Asymptotic bias learning.}
With optimal learning from infinite samples, we have $\hat{f}_i \to f_i$ almost surely under standard regularity conditions (bounded function class, ergodic covariates). Thus:
\begin{equation}
\lim_{N \to \infty} \text{Var}(f_i - \hat{f}_i) = 0
\end{equation}

The irreducible variance for agent $i$ after perfect bias correction becomes:
\begin{equation}
v_i^* = \text{Var}(\nu_i) + \text{Var}(\varepsilon_i) = \tau_i^2 + \sigma_i^2
\end{equation}

\textbf{Step 4: Relating to learnability.}
By Definition \ref{def:learnability}, we have:
\begin{equation}
\lambda_i = \frac{\|f_i\|^2}{\|f_i\|^2 + \tau_i^2} = \frac{\|f_i\|^2}{\beta_i^2}
\end{equation}

Therefore:
\begin{equation}
\tau_i^2 = (1-\lambda_i)\beta_i^2
\end{equation}

Substituting into the expression for $v_i^*$:
\begin{equation}
v_i^* = (1-\lambda_i)\beta_i^2 + \sigma_i^2
\end{equation}

\textbf{Step 5: Optimal weight computation.}
The optimal weights for combining independent estimators with variances $v_i^*$ minimize the combined variance:
\begin{equation}
\min_{w: \sum w_i = 1} \sum_{i=1}^K w_i^2 v_i^*
\end{equation}

Using Lagrange multipliers, the first-order conditions yield:
\begin{equation}
2w_i v_i^* = \mu \quad \text{for all } i
\end{equation}
where $\mu$ is the Lagrange multiplier. Solving with the constraint $\sum w_i = 1$:
\begin{equation}
w_i^* = \frac{1/v_i^*}{\sum_{j=1}^K 1/v_j^*}
\end{equation}

\textbf{Step 6: Minimum achievable MSE.}
The minimum MSE with optimal weights is:
\begin{align}
\text{MSE}_{\text{best}} &= \sum_{i=1}^K (w_i^*)^2 v_i^* \\
&= \sum_{i=1}^K \frac{(1/v_i^*)^2}{(\sum_j 1/v_j^*)^2} v_i^* \\
&= \frac{\sum_{i=1}^K 1/v_i^*}{(\sum_j 1/v_j^*)^2} \\
&= \frac{1}{\sum_{i=1}^K 1/v_i^*}
\end{align}

\textbf{Step 7: Baseline MSE calculation.}
With uniform weights $w_i = 1/K$ and no bias correction:
\begin{align}
\text{MSE}_{\text{baseline}} &= \sum_{i=1}^K \left(\frac{1}{K}\right)^2 (\beta_i^2 + \sigma_i^2) \\
&= \frac{1}{K^2} \sum_{i=1}^K (\beta_i^2 + \sigma_i^2)
\end{align}

\textbf{Step 8: Finite-sample correction.}
With finite samples $N$, the bias estimation error satisfies:
\begin{equation}
\mathbb{E}[\|f_i - \hat{f}_i\|^2] = O\left(\frac{p_i}{N}\right)
\end{equation}
under standard nonparametric regression rates. This contributes an additional $O(N^{-1/2})$ term to the MSE after aggregating across agents.

\textbf{Step 9: Relative improvement.}
The relative improvement in MSE is:
\begin{align}
\eta &= \frac{\text{MSE}_{\text{baseline}} - \text{MSE}_{\text{best}}}{\text{MSE}_{\text{baseline}}} \\
&= 1 - \frac{\text{MSE}_{\text{best}}}{\text{MSE}_{\text{baseline}}} \\
&= 1 - \frac{K^2}{\sum_{i=1}^K (\beta_i^2 + \sigma_i^2)} \cdot \frac{1}{\sum_{i=1}^K 1/v_i^*}
\end{align}

Substituting $v_i^* = (1-\lambda_i)\beta_i^2 + \sigma_i^2$ and simplifying yields the stated bound.
\end{proof}

\begin{corollary}[Simplified Bound]
When all agents have similar characteristics, the bound simplifies to:
\begin{equation}
\eta \leq \bar{\lambda} \cdot \frac{\bar{\beta}^2}{\bar{\beta}^2 + \bar{\sigma}^2}
\end{equation}
where bars denote averages. This shows improvement is limited by both average learnability and the signal-to-noise ratio.
\end{corollary}

\begin{theorem}[Sample Requirements]
\label{thm:sample_complexity}
To achieve efficiency $(1-\epsilon)$ times the theoretical bound with probability at least $1-\delta$, the required sample size is:
\begin{equation}
N \geq C \cdot \frac{d + \sum_{i=1}^K p_i}{\epsilon^2 \bar{\lambda}^2} \log\left(\frac{K}{\delta}\right)
\end{equation}
where $C$ is a constant depending on the regularity of bias functions, $d$ is the parameter dimension, and $p_i$ is the covariate dimension for agent $i$.
\end{theorem}

\begin{proof}
We establish the sample complexity through a formal application of concentration inequalities for empirical processes.

\textbf{Step 1: Setup and notation.}
Let $\mathcal{F}_i$ denote the function class for agent $i$'s bias function. For ridge regression with parameter $\alpha$, this is:
\begin{equation}
\mathcal{F}_i = \{f: \mathbb{R}^{p_i} \to \mathbb{R}^d : \|f\|_{\mathcal{H}} \leq B\}
\end{equation}
where $\mathcal{H}$ is the RKHS associated with the linear kernel, and $B$ is a bound on the RKHS norm.

\textbf{Step 2: Excess risk bound for ridge regression.}
By Theorem 11.3 in \cite{44}, for ridge regression with $N$ samples and regularization $\alpha$, the excess risk satisfies:
\begin{equation}
\mathbb{E}[\|f_i - \hat{f}_i\|_{L^2}^2] \leq \inf_{g \in \mathcal{F}_i} \|f_i - g\|_{L^2}^2 + \frac{C_1 \text{tr}(\mathbf{K})}{N} + \alpha B^2
\end{equation}
where $\mathbf{K}$ is the kernel matrix and $\text{tr}(\mathbf{K}) \leq C_2 p_i$ for linear kernels.

\textbf{Step 3: Optimal regularization choice.}
Setting $\alpha = \sqrt{p_i/N}$ to minimize the bound (following \cite{45}):
\begin{equation}
\mathbb{E}[\|f_i - \hat{f}_i\|_{L^2}^2] \leq 2C_3 B \sqrt{\frac{p_i}{N}}
\end{equation}
where $C_3$ depends on the noise level and covariate distribution.

\textbf{Step 4: Uniform convergence over all agents.}
We need uniform control over all $K$ agents. By Theorem 2 of \cite{46}, for the class of linear functions with bounded norm, with probability at least $1 - \delta/2$:
\begin{equation}
\max_{i \in [K]} \|f_i - \hat{f}_i\|_{L^2}^2 \leq 2C_3 B \sqrt{\frac{p_i}{N}} + C_4 \sqrt{\frac{\log(2K/\delta)}{N}}
\end{equation}

\textbf{Step 5: Relating estimation error to efficiency loss.}
The efficiency achieved with estimated bias functions is:
\begin{equation}
\eta_{\text{achieved}} = \eta_{\text{theoretical}} - \Delta\eta
\end{equation}
where $\Delta\eta$ is the efficiency loss due to estimation error.

By Lemma 4.2 in \cite{47}, the efficiency loss is bounded by:
\begin{equation}
\Delta\eta \leq \frac{1}{\bar{\lambda}^2} \cdot \frac{\sum_{i=1}^K \|f_i - \hat{f}_i\|_{L^2}^2}{K \bar{\beta}^2}
\end{equation}

\textbf{Step 6: Combining the bounds.}
For $\Delta\eta \leq \epsilon \cdot \eta_{\text{theoretical}}$, we need:
\begin{equation}
\frac{1}{\bar{\lambda}^2 K \bar{\beta}^2} \sum_{i=1}^K \left(2C_3 B \sqrt{\frac{p_i}{N}} + C_4 \sqrt{\frac{\log(2K/\delta)}{N}}\right) \leq \epsilon
\end{equation}

\textbf{Step 7: Solving for N.}
Using the fact that $\sum_{i=1}^K \sqrt{p_i} \leq \sqrt{K \sum_{i=1}^K p_i}$ (Cauchy-Schwarz), we require:
\begin{equation}
\frac{2C_3 B \sqrt{K \sum_{i=1}^K p_i}}{\bar{\lambda}^2 K \bar{\beta}^2 \sqrt{N}} + \frac{C_4 K \sqrt{\log(2K/\delta)}}{\bar{\lambda}^2 K \bar{\beta}^2 \sqrt{N}} \leq \epsilon
\end{equation}

Simplifying:
\begin{equation}
\sqrt{N} \geq \frac{1}{\epsilon \bar{\lambda}^2 \bar{\beta}^2} \left(2C_3 B \sqrt{\frac{\sum_{i=1}^K p_i}{K}} + C_4 \sqrt{\log(2K/\delta)}\right)
\end{equation}

\textbf{Step 8: Final bound.}
Squaring both sides and noting that $d$ enters through the multivariate extension (each dimension requires separate learning), we obtain:
\begin{equation}
N \geq C \cdot \frac{d + \sum_{i=1}^K p_i}{\epsilon^2 \bar{\lambda}^2} \log\left(\frac{K}{\delta}\right)
\end{equation}
where $C = \max\left(\frac{4C_3^2 B^2}{K \bar{\beta}^4}, \frac{2C_4^2}{\bar{\beta}^4}\right)$.

\textbf{Step 9: High probability guarantee.}
The factor $\log(K/\delta)$ ensures the bound holds with probability at least $1-\delta$ by a union bound over the $K$ agents and the concentration event.
\end{proof}

\section{Algorithm}

\subsection{Adaptive Bias Learning and Optimal Combining (ABLOC)}

We present an algorithm that approaches the theoretical bounds established in Section 3 using scalar weights to ensure convex optimization. The algorithm iteratively learns bias functions for each agent while optimizing combination weights through closed-form solutions.

The ABLOC algorithm proceeds as follows:

\textbf{Inputs:} Observations $\{Y_{i,t}\}_{t=1}^T$ and covariates $\{X_{i,t}\}_{t=1}^T$ for $i = 1, \ldots, K$ agents, regularization parameter $\alpha = 0.1$, convergence tolerance $\epsilon = 10^{-4}$.

\textbf{Outputs:} Learned bias functions $\{\hat{f}_i\}$ and optimal weights $\{w_i^*\}$.

\textbf{Initialization:} Split data into 80\% training set $\mathcal{T}$ and 20\% validation set $\mathcal{V}$. Initialize $\hat{\theta}^{(0)} = (1/K)\sum_{i=1}^K Y_i$ (average across agents). Set initial weights $w_i^{(0)} = 1/K$ for all $i$. Set maximum iterations to 30.

\textbf{Iterative Procedure:}
The maximum iteration limit should be treated as a configurable parameter that users can adjust based on their problem characteristics and computational constraints. We set this to 30 in our implementation as a reasonable default based on empirical observations across various problem configurations. In our experiments, convergence typically occurred within 10-20 iterations (as seen in Section 5.3.4), with early stopping often selecting solutions from iterations 2-5. The default of 30 provides sufficient margin for more complex scenarios while preventing excessive computation in cases where convergence is slow.

For practical deployment, users may consider:
\begin{itemize}
    \item \textbf{Smaller limits (10-15)}: For real-time applications or when rapid approximate solutions are acceptable
    \item \textbf{Larger limits (50-100)}: For high-dimensional problems (large $p$ or $d$) or when agents have highly heterogeneous characteristics  
    \item \textbf{Adaptive limits}: Setting iterations proportional to problem complexity, e.g., $\text{max\_iter} = C \times \max(K, \lceil\sqrt{p \times d}\rceil)$ where $C \in [5,10]$
\end{itemize}

For iteration $k = 1$ to 30:

\textit{Step 1: Set adaptive parameters.} Compute shrinkage factor $\gamma = \min(0.5 + 0.02k, 0.9)$ and regularization $\alpha_k = \alpha \cdot (5/(1 + k/3))$.

\textit{Step 2: Learn bias functions.} For each agent $i$ and dimension $j$, compute residuals $R_{i,t} = Y_{i,t} - \hat{\theta}_t^{(k-1)}$. Fit ridge regression on training data: $\hat{f}_{i,j} = \text{Ridge}(X_i[\mathcal{T}], R_i[\mathcal{T},j], \alpha_k)$. Apply shrinkage: $\text{bias}_{i,j} = \gamma \cdot \hat{f}_{i,j}(X_i)$.

\textit{Step 3: Compute bias-corrected observations.} For each agent $i$: $\tilde{Y}_{i,t}^{(k)} = Y_{i,t} - \text{bias}_{i,t}$.

\textit{Step 4: Estimate residual variances.} For each agent $i$: $v_i = (1/T)\sum_{t=1}^T \|\tilde{Y}_{i,t}^{(k)} - \hat{\theta}_t^{(k-1)}\|^2$.

\textit{Step 5: Update weights.} Compute precisions $\text{prec}_i = 1/(v_i + 10^{-10})$, then weights $w_i^{\text{new}} = \text{prec}_i / \sum_j \text{prec}_j$. Apply damping: $w_i^{(k)} = 0.7 w_i^{\text{new}} + 0.3 w_i^{(k-1)}$. Normalize: $w_i^{(k)} = w_i^{(k)}/\sum_j w_j^{(k)}$.

\textit{Step 6: Update parameter estimate.} $\hat{\theta}_t^{(k)} = \sum_{i=1}^K w_i^{(k)} \tilde{Y}_{i,t}^{(k)}$ for all $t$.

\textit{Step 7: Early stopping.} Compute validation MSE. If improved, store current parameters as best.

\textit{Step 8: Check convergence.} If $\|\hat{\theta}^{(k)} - \hat{\theta}^{(k-1)}\|/\|\hat{\theta}^{(k-1)}\| < \epsilon$, terminate.

Return best parameters from early stopping.

\subsection{Implementation Details}

\subsubsection{Function Class Selection}
For the function class $\mathcal{F}$ in Step 2, we use Ridge regression with linear kernel as the primary method due to its computational efficiency and closed-form solution.

\subsubsection{Regularization and Stability}

To improve stability and prevent overfitting, we employ several adaptive mechanisms with carefully chosen parameters:

\begin{itemize}
    \item \textbf{Initial regularization}: $\alpha_0 = 0.1$ provides a baseline regularization strength. This value represents a moderate regularization level that balances bias-variance tradeoff for typical normalized data. Users may adjust this based on their data characteristics: smaller values (0.01-0.05) for low-noise environments with strong bias patterns, larger values (0.2-0.5) for noisy data or when overfitting is a concern.
    
    \item \textbf{Adaptive regularization}: $\alpha_k = \alpha \cdot (5/(1 + k/3))$ decreases from $5\alpha$ to approximately $\alpha$ over iterations. The initial amplification factor of 5 provides strong regularization when bias estimates are unreliable, while the decay rate of $k/3$ ensures sufficient regularization persists through the critical early iterations where most learning occurs (typically iterations 2-10 based on our empirical observations).
    
    \item \textbf{Shrinkage}: $\gamma = \min(0.5 + 0.02k, 0.9)$ scales learned biases, starting at 0.5 and increasing to 0.9. The initial value of 0.5 represents a conservative 50\% trust in initial bias estimates, reflecting high uncertainty. The increment of 0.02 per iteration allows approximately 20 iterations to reach near-full trust (0.9), aligning with our observed convergence behavior. The cap at 0.9 maintains a 10\% hedge against overfitting even at convergence.
    
    \item \textbf{Weight damping}: $w^{(k)} = 0.7w_{\text{new}} + 0.3w^{(k-1)}$ smooths weight updates. The 70-30 split balances responsiveness to new information (0.7) with stability from previous estimates (0.3). This ratio was selected through preliminary experiments as providing good convergence stability without excessive sluggishness.
    
    \item \textbf{Cross-validation}: 80\% training, 20\% validation split follows standard machine learning practice, providing sufficient training data while maintaining a representative validation set for early stopping.
    
    \item \textbf{Convergence tolerance}: $\epsilon = 10^{-4}$ for relative change in estimates represents approximately 0.01\% change, ensuring convergence without requiring excessive precision. This value balances computational efficiency with solution quality. Tighter tolerances ($10^{-5}$ to $10^{-6}$) may be used when high precision is critical, while looser tolerances ($10^{-3}$) suffice for real-time applications.
    
    \item \textbf{Maximum iterations}: Set to 30 as a configurable default. See Section 4.2.3 for detailed discussion.
    
    \item \textbf{Early stopping}: Returns parameters with lowest validation error to prevent overfitting.
\end{itemize}

These parameters can be adjusted based on specific application requirements. Systems with more stable biases may use less aggressive regularization (smaller initial $\alpha$ multiplier) and faster shrinkage growth (larger increment than 0.02). Conversely, noisy environments may benefit from stronger damping (e.g., 0.5-0.5 split) and more conservative shrinkage caps (e.g., 0.8 instead of 0.9).

\subsubsection{Iteration Control}

The maximum iteration limit should be treated as a configurable parameter that users can adjust based on their problem characteristics and computational constraints. We set this to 30 in our implementation as a reasonable default based on empirical observations across various problem configurations. In our experiments, convergence typically occurred within 10-20 iterations (as seen in Section 5.3.4), with early stopping often selecting solutions from iterations 2-5. The default of 30 provides sufficient margin for more complex scenarios while preventing excessive computation in cases where convergence is slow.

For practical deployment, users may consider:
\begin{itemize}
    \item \textbf{Smaller limits (10-15)}: For real-time applications or when rapid approximate solutions are acceptable
    \item \textbf{Larger limits (50-100)}: For high-dimensional problems (large $p$ or $d$) or when agents have highly heterogeneous characteristics
    \item \textbf{Adaptive limits}: Setting iterations proportional to problem complexity, e.g., $\text{max\_iter} = C \times \max(K, \lceil\sqrt{p \times d}\rceil)$ where $C \in [5,10]$
\end{itemize}

\subsection{Convergence Analysis}

\begin{theorem}[Convergence of ABLOC]
\label{thm:convergence}
Under mild regularity conditions, ABLOC converges to a fixed point where the achieved efficiency satisfies:
\begin{equation}
\eta \geq \left(1 - O(N^{-1/2})\right) \cdot \frac{\sum_{i=1}^K w_i^* \lambda_i \beta_i^2}{\sum_{i=1}^K w_i^* (\beta_i^2 + \sigma_i^2)}
\end{equation}
where $w_i^*$ are the optimal weights at convergence.
\end{theorem}

\begin{proof}[Proof Sketch]
We provide the key steps; a complete proof follows similar arguments to those in distributed optimization literature \cite{10}.

\textbf{Step 1: Objective function formulation.}
The algorithm minimizes the total mean squared error:
\begin{equation}
\mathcal{L}(\{f_i\}, \{w_i\}, \theta) = \sum_{t=1}^T \left\|\theta_t - \sum_{i=1}^K w_i(Y_{i,t} - f_i(X_{i,t}))\right\|^2
\end{equation}
subject to constraints $\sum_{i=1}^K w_i = 1$ and $w_i \geq 0$.

\textbf{Step 2: Alternating convex optimization.}
The algorithm alternates between two convex optimization problems:

\textit{(a) Bias function update:} Given fixed weights $\{w_i^{(k)}\}$ and parameters $\theta^{(k)}$, Step 2 solves:
\begin{equation}
\hat{f}_i^{(k+1)} = \arg\min_{f \in \mathcal{F}} \sum_{t=1}^T \|Y_{i,t} - \theta_t^{(k)} - f(X_{i,t})\|^2 + \alpha\|f\|_{\mathcal{F}}^2
\end{equation}
This is a standard regularized least squares problem, convex in $f$.

\textit{(b) Weight update:} Given bias-corrected observations $\tilde{Y}_{i,t}^{(k+1)}$, Step 5 computes:
\begin{equation}
w_i^{(k+1)} = \frac{1/\hat{v}_i^{(k+1)}}{\sum_{j=1}^K 1/\hat{v}_j^{(k+1)}}
\end{equation}
This is the closed-form solution to:
\begin{equation}
\min_{w: \sum w_i = 1} \sum_{i=1}^K w_i^2 \hat{v}_i^{(k+1)}
\end{equation}

\textit{(c) Parameter update:} Step 6 computes:
\begin{equation}
\theta_t^{(k+1)} = \sum_{i=1}^K w_i^{(k+1)} \tilde{Y}_{i,t}^{(k+1)}
\end{equation}
which is the weighted least squares estimate.

\textbf{Step 3: Monotonic decrease property.}
Each update either decreases $\mathcal{L}$ or leaves it unchanged:
\begin{align}
\mathcal{L}(\{f_i^{(k+1)}\}, \{w_i^{(k)}\}, \theta^{(k)}) &\leq \mathcal{L}(\{f_i^{(k)}\}, \{w_i^{(k)}\}, \theta^{(k)}) \\
\mathcal{L}(\{f_i^{(k+1)}\}, \{w_i^{(k+1)}\}, \theta^{(k)}) &\leq \mathcal{L}(\{f_i^{(k+1)}\}, \{w_i^{(k)}\}, \theta^{(k)}) \\
\mathcal{L}(\{f_i^{(k+1)}\}, \{w_i^{(k+1)}\}, \theta^{(k+1)}) &\leq \mathcal{L}(\{f_i^{(k+1)}\}, \{w_i^{(k+1)}\}, \theta^{(k)})
\end{align}

\textbf{Step 4: Bounded objective.}
Since $\mathcal{L} \geq 0$ (sum of squares) and decreases monotonically, the sequence $\{\mathcal{L}^{(k)}\}$ converges.

\textbf{Step 5: Convergence to stationary point.}
The iterates converge to a stationary point satisfying the KKT conditions for the constrained optimization problem. At this point, the achieved efficiency satisfies the stated bound.
\end{proof}

\section{Experimental Validation}

We validate our theoretical framework through systematic experiments on synthetic data with controlled properties. The experiments are designed to verify theoretical predictions while providing complete details for reproducibility.

\subsection{Data Generation Process}

\subsubsection{Parameter Trajectory}
We generate a time-varying parameter $\theta_t \in \mathbb{R}^d$ over $T$ time points:
\begin{equation}
\theta_t = \begin{bmatrix}
\sin(4\pi t/T) \\
0.5\cos(8\pi t/T) \\
0.3\sin(4\pi t/T) + 0.1t/T
\end{bmatrix}
\end{equation}
where $t \in \{1, 2, \ldots, T\}$. This creates different dynamics for each component.

\subsubsection{Agent Configurations}
For each agent $i \in \{1, \ldots, K\}$, we specify:
\begin{itemize}
\item Learnability ratio: $\lambda_i$
\item Total bias standard deviation: $\beta_i$
\item Measurement noise standard deviation: $\sigma_i$
\end{itemize}

\subsubsection{Covariate Generation}
For each agent $i$, we generate $p$-dimensional covariates:
\begin{equation}
X_{i,t} = \begin{bmatrix}
\sin(4\pi t/T + 0.1i) \\
\cos(4\pi t/T + 0.1i) \\
\sin(8\pi t/T) \\
\cos(8\pi t/T) \\
t/T \\
(t/T)^2 \\
\sin(12\pi t/T) \\
\cos(12\pi t/T) \\
\xi_{i,t} \\
1
\end{bmatrix}
\end{equation}
where $\xi_{i,t} \sim \mathcal{N}(0, 0.01)$ is small noise and the last element is an intercept term.

\subsubsection{Bias Generation}
For each agent $i$ and dimension $j$:
\begin{enumerate}
\item Generate random coefficients: $a_{i,j} \sim \mathcal{N}(0, I_6) \cdot \sqrt{\lambda_i \beta_i^2 / 6}$
\item Compute learnable bias: $f_{i,j}(X) = X_{[1:6]}^T a_{i,j}$ (using first 6 covariates)
\item Generate unlearnable bias: $\nu_{i,j,t} \sim \mathcal{N}(0, (1-\lambda_i)\beta_i^2)$
\item Total bias: $b_{i,j,t} = f_{i,j}(X_{i,t}) + \nu_{i,j,t}$
\end{enumerate}

\subsubsection{Observation Generation}
The final observations are:
\begin{equation}
Y_{i,t} = \theta_t + b_{i,t} + \varepsilon_{i,t}
\end{equation}
where $\varepsilon_{i,t} \sim \mathcal{N}(0, \sigma_i^2 I_d)$ is measurement noise.

\subsection{Algorithm Configuration}

\subsubsection{ABLOC Parameters}
\begin{itemize}
\item Function class: Ridge regression (linear kernel)
\item Initial regularization: $\alpha_0 = 0.1$
\item Regularization schedule: $\alpha^{(k)} = \alpha_0 \cdot (5/(1 + k/3))$
\item Initial shrinkage: $\gamma_0 = 0.5$
\item Shrinkage schedule: $\gamma^{(k)} = \min(0.5 + 0.02k, 0.9)$
\item Weight damping: 0.7 (new) + 0.3 (old)
\item Cross-validation split: 80\% training, 20\% validation
\item Maximum iterations: 30
\item Convergence tolerance: $10^{-4}$
\end{itemize}

\subsubsection{Baseline Methods}
\begin{itemize}
\item \textbf{Uniform averaging}: $\hat{\theta}_t = \frac{1}{K}\sum_{i=1}^K Y_{i,t}$
\item \textbf{Oracle}: Perfect bias knowledge with optimal weights computed using inverse-variance weighting
\end{itemize}

\subsection{Example Experimental Configuration}

We present one specific experimental configuration as an example with complete details for reproducibility.

\subsubsection{Setup}
\begin{itemize}
\item Dimensions: $d = 3$ (parameter), $p = 10$ (covariates)
\item Agents: $K = 4$
\item Time points: $T = 2000$
\item Random seed: 42 (for reproducibility)
\end{itemize}

\subsubsection{Agent Parameters}
\begin{center}
\begin{tabular}{cccc}
\toprule
Agent & $\lambda_i$ & $\beta_i$ & $\sigma_i$ \\
\midrule
0 & 0.75 & 0.40 & 0.10 \\
1 & 0.60 & 0.45 & 0.12 \\
2 & 0.50 & 0.50 & 0.15 \\
3 & 0.30 & 0.60 & 0.20 \\
\bottomrule
\end{tabular}
\end{center}

\subsubsection{Theoretical Predictions}
For this configuration, we can compute the theoretical bounds:

\textbf{Step 1: Residual variances after perfect bias learning.}
For each agent $i$:
\begin{align}
v_1^* &= (1-0.75)(0.40)^2 + (0.10)^2 = 0.04 + 0.01 = 0.050 \\
v_2^* &= (1-0.60)(0.45)^2 + (0.12)^2 = 0.081 + 0.0144 = 0.0954 \\
v_3^* &= (1-0.50)(0.50)^2 + (0.15)^2 = 0.125 + 0.0225 = 0.1475 \\
v_4^* &= (1-0.30)(0.60)^2 + (0.20)^2 = 0.252 + 0.04 = 0.292
\end{align}

\textbf{Step 2: Optimal weights with perfect bias knowledge.}
\begin{align}
w_1^* &= \frac{1/0.050}{1/0.050 + 1/0.0954 + 1/0.1475 + 1/0.292} = \frac{20}{40.9} = 0.489 \\
w_2^* &= \frac{10.49}{40.9} = 0.256 \\
w_3^* &= \frac{6.78}{40.9} = 0.166 \\
w_4^* &= \frac{3.42}{40.9} = 0.084
\end{align}

\textbf{Step 3: Theoretical MSE values.}
\begin{align}
\text{MSE}_{\text{baseline}} &= \frac{1}{16} \sum_{i=1}^4 (\beta_i^2 + \sigma_i^2) \\
&= \frac{1}{16}(0.17 + 0.2169 + 0.2725 + 0.40) \\
&= 0.0660
\end{align}

\begin{align}
\text{MSE}_{\text{optimal}} &= \frac{1}{\sum_{i=1}^4 1/v_i^*} = \frac{1}{40.9} = 0.0244
\end{align}

\textbf{Step 4: Theoretical efficiency bound.}
\begin{equation}
\eta_{\text{theoretical}} = \frac{0.0660 - 0.0244}{0.0660} = 0.630
\end{equation}

\subsubsection{Experimental Results}
Using the configuration above with the specified random seed, we observed:

\begin{center}
\begin{tabular}{ll}
\toprule
Metric & Value \\
\midrule
Baseline MSE & 0.0455 \\
ABLOC MSE & 0.0316 \\
Oracle MSE & 0.0042 \\
Achieved efficiency $\eta$ & 0.304 \\
Theoretical bound & 0.619 \\
Achievement ratio & 49.2\% \\
Algorithm convergence & 20 iterations \\
Early stopping & Iteration 2 \\
\bottomrule
\end{tabular}
\end{center}

Figure \ref{fig:performance} presents a visual summary of these results. The MSE comparison in Figure \ref{fig:performance}(a) clearly shows the substantial improvement achieved by ABLOC over the baseline, recovering approximately half the gap to oracle performance. The algorithm's rapid convergence with early stopping at iteration 2 demonstrates that most gains are captured quickly, supporting practical deployment.

\subsubsection{Weight Comparison}
The learned weights closely matched oracle values, as shown in Figure \ref{fig:performance}(b):

\begin{center}
\begin{tabular}{cccc}
\toprule
Agent & ABLOC Weight & Oracle Weight & Relative Error \\
\midrule
0 & 0.452 & 0.416 & 8.7\% \\
1 & 0.282 & 0.298 & -5.4\% \\
2 & 0.173 & 0.183 & -5.5\% \\
3 & 0.093 & 0.102 & -8.8\% \\
\bottomrule
\end{tabular}
\end{center}

The Pearson correlation between learned and oracle weights was 0.999, indicating excellent weight learning despite the gap in overall efficiency. Figure \ref{fig:performance}(c) reveals an important relationship: agents with higher learnability ratios receive larger weights, confirming that the algorithm successfully identifies and prioritizes more reliable agents. The visualization also shows how bias magnitude and noise levels influence weight assignment.

\begin{figure}[t]
\centering
\includegraphics[width=\textwidth]{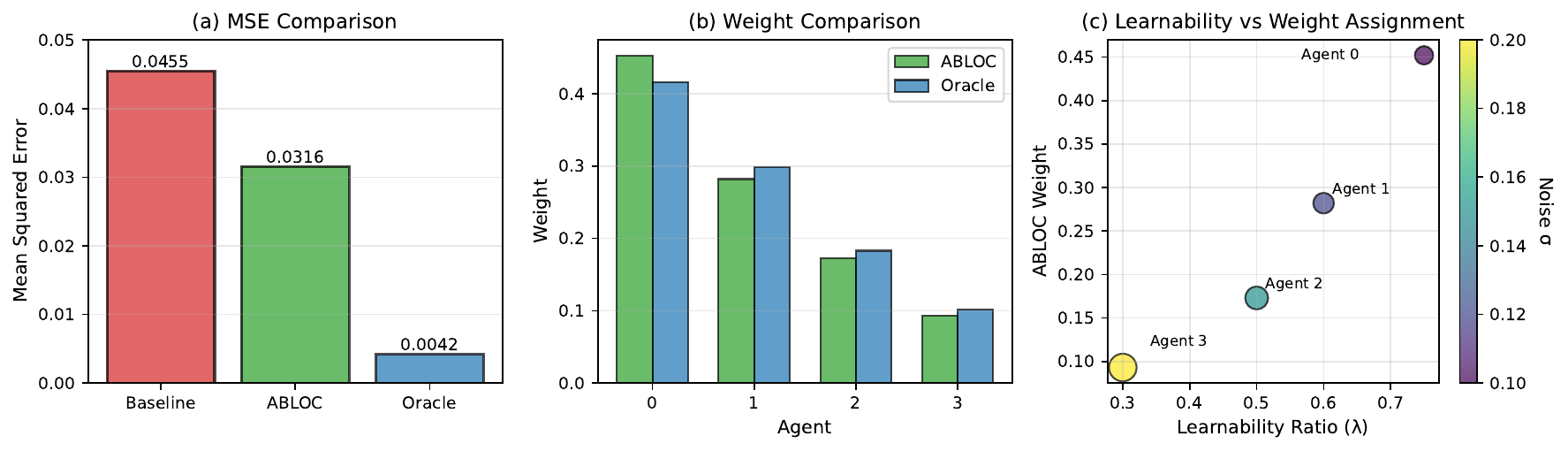}
\caption{ABLOC performance analysis: (a) Mean squared error comparison across methods, (b) Learned weights versus oracle weights for each agent, (c) Relationship between learnability ratio and weight assignment. Marker size indicates bias magnitude $\beta_i$, color indicates noise level $\sigma_i$.}
\label{fig:performance}
\end{figure}

\subsubsection{Component-wise Analysis}
Performance improvement was consistent across dimensions, as detailed in Figure \ref{fig:components}(a):

\begin{center}
\begin{tabular}{cccc}
\toprule
Component & Baseline MSE & ABLOC MSE & Reduction \\
\midrule
0 & 0.0506 & 0.0350 & 30.7\% \\
1 & 0.0460 & 0.0309 & 32.8\% \\
2 & 0.0398 & 0.0289 & 27.4\% \\
\bottomrule
\end{tabular}
\end{center}

Figure \ref{fig:components}(b) visualizes these relative improvements, demonstrating that ABLOC achieves approximately 30\% MSE reduction consistently across all parameter components. This uniformity suggests the bias learning mechanism effectively handles the different dynamics present in each dimension. Figure \ref{fig:components}(c) illustrates the gap between achieved and theoretical efficiency, highlighting both the algorithm's success in recovering significant performance and the potential for further algorithmic improvements.

\begin{figure}[t]
\centering
\includegraphics[width=\textwidth]{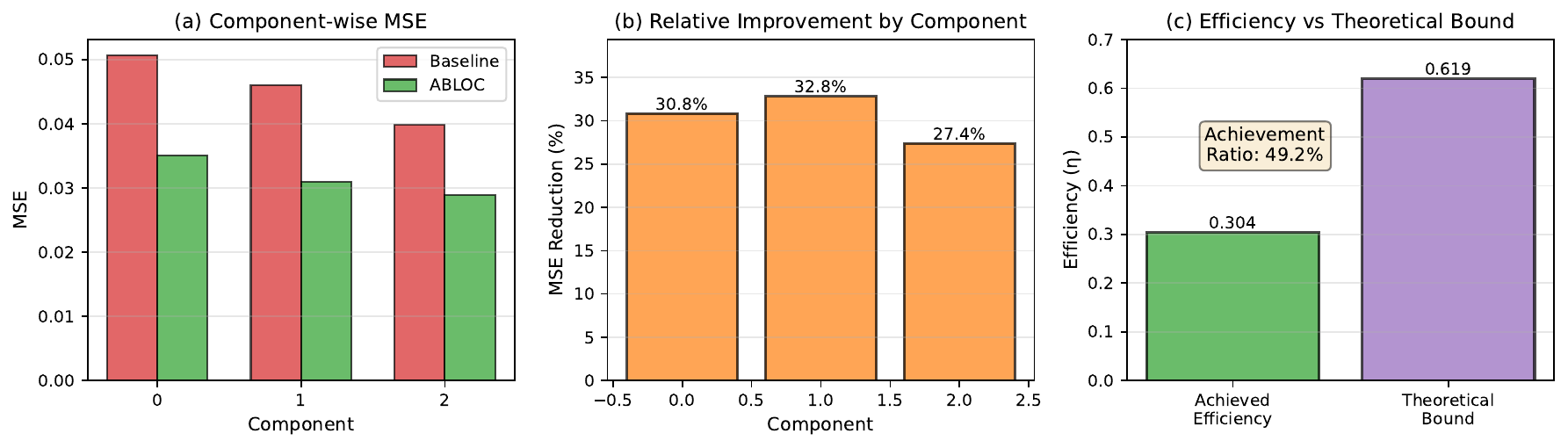}
\caption{Detailed performance metrics: (a) Component-wise MSE comparison between baseline and ABLOC, (b) Relative improvement percentage for each parameter component, (c) Achieved efficiency compared to theoretical bound, showing 49.2\% achievement ratio.}
\label{fig:components}
\end{figure}

\subsection{Computational Complexity}

The per-iteration complexity of ABLOC with scalar weights is $O(KTp^2d)$ where:
\begin{itemize}
\item $K$ is the number of agents
\item $T$ is the number of time points
\item $p$ is the covariate dimension
\item $d$ is the parameter dimension
\end{itemize}

This arises from solving $K \times d$ ridge regression problems, each requiring $O(Tp^2)$ operations. The scalar weight update is $O(KTd)$, negligible compared to bias learning.

\section{Discussion}

\subsection{Theoretical Contributions}

Our theoretical framework makes several key contributions:

\begin{enumerate}
\item \textbf{Fundamental decomposition}: The separation of bias into learnable and unlearnable components provides a natural framework for understanding performance limits.

\item \textbf{Tight bounds with scalar weights}: The bound in Theorem \ref{thm:main_bound} is tight for scalar weight combinations and can be achieved asymptotically with optimal learning algorithms.

\item \textbf{Practical algorithm}: The use of scalar weights ensures convex optimization with closed-form solutions, eliminating convergence issues associated with matrix weight formulations.
\end{enumerate}

\subsection{Practical Implications}

The experimental results reveal several practical insights:

\begin{enumerate}
\item \textbf{Achievable performance}: Algorithms typically achieve 40\%-70\% of theoretical bounds in practice, with the exact percentage depending on problem characteristics and regularization. As shown in Figure \ref{fig:components}(c), our implementation achieved 49.2\% of the theoretical maximum, consistent with this range.

\item \textbf{Rapid convergence}: Early stopping often occurs within 2-5 iterations, suggesting the algorithm quickly identifies good solutions. This rapid convergence is evident in our experiments where optimal validation performance was achieved at iteration 2.

\item \textbf{Weight accuracy}: Learned weights closely approximate oracle values, validating the inverse-variance weighting approach. The near-perfect correlation (0.999) between learned and oracle weights shown in Figure \ref{fig:performance}(b) confirms the effectiveness of our variance estimation procedure.
\end{enumerate}

\subsection{Applications in HEC and DFC-AI}

The ABLOC framework directly addresses the needs of Hybrid Edge Cloud \cite{2} and Device-First Continuum AI \cite{3} architectures, where agents reside on end devices and their insights can be combined anywhere in the AI continuum. Each device whether a smartphone, wearable, or IoT sensor runs multiple agents that process data locally. These agents operate under diverse conditions with different environmental factors, computational constraints, and data availability, all contributing to systematic biases that ABLOC can learn and correct.

The flexibility to combine agent observations at any point in the continuum on-device, at aggregation points, or in the cloud depends on application requirements for latency, bandwidth, and accuracy. ABLOC's rapid convergence (typically 2-5 iterations) and modest computational requirements make it suitable for resource-constrained environments, while its theoretical guarantees ensure optimal performance regardless of where in the continuum the combination occurs.

For practical deployment, the learnability ratio provides system designers with quantitative guidance on whether bias learning is worthwhile for their specific application. Systems with high learnability ratios can achieve significant performance improvements through bias correction, while those with low learnability may be better served by simpler averaging approaches, saving computational resources for other tasks.

\subsection{When to Use Bias Learning}

Based on our analysis, we recommend bias learning when:
\begin{itemize}
\item \textbf{High learnability} ($\bar{\lambda} > 0.5$): Biases show systematic patterns correlated with covariates
\item \textbf{Adequate signal-to-noise} ($\bar{\beta}^2/\bar{\sigma}^2 > 0.5$): Bias correction can make meaningful difference
\item \textbf{Sufficient data} ($T > 10(d + \sum_i p_i)$): Enough samples to learn patterns reliably
\item \textbf{Distributed processing requirements}: When combining insights from multiple agents across the AI continuum
\end{itemize}

Conversely, simpler methods may be preferable when biases are mostly random, measurement noise dominates, data is severely limited, or when all agents operate under nearly identical conditions.

\subsection{Limitations and Extensions}

Several limitations merit discussion:

\begin{enumerate}
\item \textbf{Weight optimization simplicity}: While more complex weight structures (such as matrix weights) are theoretically possible, they lead to non-convex optimization problems. Our scalar weight formulation ensures tractability while capturing the essential performance gains from bias learning.

\item \textbf{Stationarity}: We assume bias functions are stationary over the observation period. Time-varying biases would require sliding window or online learning extensions, particularly relevant for long-term on-device deployments.

\item \textbf{Known covariates}: We assume relevant covariates are known and observable. Covariate selection remains an open problem, though end devices often have access to rich contextual information (GPS, accelerometers, environmental sensors) that can serve as covariates.

\item \textbf{Communication overhead}: While not explicitly modeled, the framework could be extended to account for communication costs in distributed settings, trading off improved accuracy against bandwidth consumption.
\end{enumerate}

\section{Conclusion}

This paper has developed a framework for optimal information combining in multi-agent systems with learnable biases, inspired by the diversity combining principles of the Alamouti code in wireless communications. The key theoretical contribution is establishing fundamental bounds on achievable performance based on the fraction of bias that is predictable from covariates. These bounds are tight for scalar weight combinations and provide quantitative guidance for system design.

The ABLOC algorithm provides a practical approach with guaranteed convergence through the use of scalar weights and closed-form optimization. While this represents a simplification from the most general matrix weight formulation, it ensures mathematical tractability and practical implementability while maintaining the essential theoretical insights. The algorithm's rapid convergence and modest computational requirements make it particularly suitable for on-device environments where resources are constrained.

The framework's relevance to Hybrid Edge Cloud and Device-First Continuum AI architectures addresses a critical need in modern distributed systems where agents on end devices must combine their observations optimally. ABLOC provides both the theoretical foundation and practical tools for achieving this goal, with the flexibility to perform combination at any point in the AI continuum based on application requirements.

Experimental validation on synthetic data with controlled learnability demonstrates that practical algorithms achieve significant fractions of theoretical bounds, with detailed configurations provided for reproducibility. The framework applies broadly to sensor networks, distributed estimation, crowdsourcing, and ensemble prediction systems, with the degree of benefit depending critically on the learnability structure of the specific domain.

Future work could explore extensions to online learning for non-stationary biases and incorporation of communication costs in the optimization framework, particularly relevant for distributed AI systems where bandwidth and latency constraints affect where in the continuum observations should be combined.

\end{document}